\documentclass{ecai}
\usepackage{times}
\usepackage{graphicx}
\usepackage{latexsym}
\usepackage{epsfig, amsthm, amsmath, amssymb}
\usepackage{multirow}
\usepackage{algorithm}
\usepackage{algorithmic}
\usepackage{theoremref}
\usepackage{xcolor}
\newtheorem{definition}{Definition}[section]
\newtheorem{thm}{Theorem}

\ecaisubmission 

\def\astrightarrow{\put(0.2,-2.5){*}\rightarrow}
\def\astleftarrow{\leftarrow\put(-3.6,-2.5){*}}

\newcommand{\eat}[1]{}

\makeatletter
\newcommand*{\indep}{%
  \mathbin{%
    \mathpalette{\@indep}{}%
  }%
}
\newcommand*{\nindep}{%
  \mathbin{
    \mathpalette{\@indep}{\not}
  }%
}
\newcommand*{\@indep}[2]{%
  \sbox0{$#1\perp\m@th$}
  \sbox2{$#1=$}
  \sbox4{$#1\vcenter{}$}
  \rlap{\copy0}
  \dimen@=\dimexpr\ht2-\ht4-.2pt\relax
  \kern\dimen@
  {#2}%
  \kern\dimen@
  \copy0 %
}
\usepackage{authblk}
\begin{document}

\title{Causal query in observational data with hidden variables}
\author{Debo Cheng\institute{School of Information Technology and Mathematical Sciences, University of South Australia, Mawson Lakes, SA 5095, Australia. E-mail: chedy055@mymail.unisa.edu.au and Jiuyong.Li@unisa.edu.au.}\hspace{3pt}, Jiuyong Li$^{1}$, Lin Liu$^{1}$, Jixue Liu$^{1}$, Kui Yu\institute{School of Computer Science and Information Engineering, Hefei University of Technology, Hefei, China.}\hspace{3pt}, and Thuc Duy Le$^{1}$ \\
}

\maketitle

\begin{abstract}
 This paper discusses the problem of causal query in observational data with hidden variables, with the aim of seeking the change of an outcome when ``manipulating'' a variable while given a set of plausible confounding variables which affect the manipulated variable and the outcome. Such an ``experiment on data'' to estimate the causal effect of the manipulated variable is useful for validating an experiment design using historical data or for exploring confounders when studying a new relationship. However, existing data-driven methods for causal effect estimation face some major challenges, including poor scalability with high dimensional data, low estimation accuracy due to heuristics used by the global causal structure learning algorithms, and the assumption of causal sufficiency when hidden variables are inevitable in data. In this paper, we develop a theorem for using local search to find a superset of the adjustment (or confounding) variables for causal effect estimation from observational data under a realistic pretreatment assumption. The theorem ensures that the unbiased estimate of causal effect is included in the set of causal effects estimated by the superset of adjustment variables.
 Based on the developed theorem, we propose a data-driven algorithm for causal query. Experiments show that the proposed algorithm is faster and produces better causal effect estimation than an existing data-driven causal effect estimation method with hidden variables. The causal effects estimated by the proposed algorithm are as accurate as those by the state-of-the-art methods using domain knowledge.
\end{abstract}

\section{Introduction}
\label{Sec:Intro}
Data is frequently used for various decision making which involves causal evidence seeking (or causal query) in observational data~\cite{pearl2009causality,li2016observational,yu2019learning,yu2019multi}. For example, a biologist is planning for a biological experiment, and she has a collection of data from previous experiments by other researchers (those experiments do not need to have the same objective as hers)~\cite{cai2013causal,le2013inferring}. She can plan the experiment on the data by ``manipulating'' the variable which is to be modified in the planned experiment, i.e. the treatment variable to see if the outcome is changed as expected. This type of query is useful in many real-world applications. When a public policy is in the design phase, existing data can be used to assess the policy to find out whether the policy would produce the desired outcome and what major factors would affect the outcome. In a marketing campaign, similar queries can be made to a dataset to help plan a successful campaign.

In this paper, we consider the problem that a user queries a dataset for a causal effect, i.e. how much the interested outcome $Y$ will change due to a change of the treatment $W$. The answer depends on what other variables affect $W$ and $Y$, called confounders or adjustment variables. Users do not know the adjustment variables in most cases. For example, when a biologist studies how a gene causes cancer, she may not know (or only partially know) the other genes which are causal factors of the cancer and regulate this gene under study. So, the adjustment variables are a part of the answer to seek too and the query answer will be a set of pairs each containing a set of possible adjustment variables (called an \emph{adjustment set} in the paper) and the corresponding causal effect. To confirm which pair corresponds to the real-world mechanism, a user can design an experiment by physically controlling the variables in an adjustment set while manipulating the treatment variable, or pick up one or multiple returned adjustment sets consistent with her knowledge of the system as plausible answers.

From data, especially data with hidden variables~\cite{cai2019triad,richardson2002ancestral,yu2018mining}, it may not be possible to find an exact adjustment set. The number of returned possible adjustment sets can be very large, which will make the query answer unusable. For example, a biologist has only a limited budget to try a few experiments, and she will not be able to try many possible adjustment sets. An advertisement company can try a limited number of A/B tests to make a promotion decision. Another constraint is the size of an adjustment set. A long adjustment set containing many variables will complicate the design of an experiment and make it hard to explain experimental results.

The causal query problem studied in this paper is related to adjustment variable search and causal effect estimation in data. However, existing works on causal effect estimation often have an adjustment set given~\cite{athey2018approximate,imai2014covariate,imbens2015causal} or do not explicitly indicate the adjustment set for the causal effect~\cite{johansson2016learning,louizos2017causal,shalit2017estimating}. Hence they are not useful to solve our problem.

Graphical causal modelling is a principled approach to adjustment set search in data and causal effect estimation. Pearl has proposed the \emph{back-door criterion} for identifying an adjustment set based on a given causal DAG (direct acyclic graph), and the \emph{do calculus} for deriving identifiable causal effects from the given causal DAG and data~\cite{pearl2009causality}. The back-door criterion has become the main principle for identifying adjustment sets.

When there are no hidden variables or more precisely the \emph{causal sufficiency} assumption is satisfied, i.e. all common causes are observed~\cite{spirtes2000causation}, we can learn from data a Markov equivalence class of causal DAGs for causal effect estimation. IDA is an algorithm for estimating causal effects directly from data~\cite{maathuis2009estimating}. As many equivalent causal DAGs can be learned from a dataset, IDA returns a multiset of causal effects with different adjustment sets blocking the back-door paths into the treatment variable $W$. Recently, Johansson et al. utilised deep learning to learn balanced representations for counterfactual inference~\cite{johansson2016learning}. Furthermore, Shalit et al. gave a meaningful and intuitive error bound to guide deep neural networks for estimating individual causal effects~\cite{shalit2017estimating}. However, the algorithms based on deep learning do not provide an explicit adjustment variable set.

When there are hidden variables, LV-IDA~\cite{malinsky2016estimating} extends IDA based to a MAG (Maximal Ancestral Graph) which represents a causal structure with hidden variables. LV-IDA also returns a multiset of causal effects. Furthermore, LV-IDA has low efficiency for high-dimensional or large data. Hyttinen et al. presented an ASP constraint solver for causal inference directly in data without assuming causal sufficiency~\cite{hyttinen2015calculus}, but it is impractical for high-dimensional data due to the fact that the query-based techniques need to perform a complete search for solutions. Entner et al. proposed a data-driven approach to estimating causal effect from data with hidden variables~\cite{entner2013data}, but with impractically high time complexity. Louizos et al. developed a new neural network latent variables model to estimate population causal effects~\cite{louizos2017causal}, but the model relies on a correct proxy variable estimation and does not provide an explicit adjustment set.

In this paper, we aim to develop a practical solution for efficient and accurate causal query in data without assuming causal sufficiency. The contributions of the work are:
\begin{itemize}
  \item We have developed the theorem for finding a superset of adjustment variables for unbiased causal effect estimation based on the local causal structure learning in a MAG under a realistic pretreatment assumption. The theorem supports efficient data-driven causal effect estimations without assuming causal sufficiency in data. To our best knowledge, they are the first theorem to support such a search.
  \item We have proposed an efficient and effective Data-drIven Causal Effect estimation (DICE) algorithm. DICE is faster than LV-IDA, the only practical data-driven causal effect estimation method with hidden variables in graphical causal modelling. Extensive experiments show that DICE works in datasets where LV-IDA fails, and produces better causal effect estimation than LV-IDA. The causal effects estimated by DICE are as accurate as those by the state-of-the-art methods using domain knowledge.
\end{itemize}

\section{Preliminaries}
\label{Sec:Pre}
\subsection{Notation and basic definitions}
We use upper case letters to represent variables and bold-faced upper case letters to denote sets of variables. Let $\mathcal{G}=(\mathbf{V}, \mathbf{E})$ be a graph, where $\mathbf{V}=\{V_{1}, \dots, V_{p}\}$ is the set of nodes and $\mathbf{E}$ a set of edges. Two nodes are adjacent if there is an edge between them. The set of all the nodes adjacent to node $V$ is denoted as $Adj(V)$. A path $\pi$ from $V_{s}$ to $V_{e}$ is a sequence of distinct nodes $<V_{s}, \dots, V_{e}>$ such that every pair of successive nodes are adjacent in $\mathcal{G}$. A sub-path of $\pi$ from $V_i$ to $V_j$ is denoted by $\pi(V_i, V_j)$. In $\mathcal{G}$, if there exists $V_i\rightarrow V_j$, $V_i$ is a parent of $V_j$ and we use $Pa(V_j)$ to denote the set of all parents of $V_j$. If there exists $V_i\leftrightarrow V_j$, $V_i$ is a spouse of $V_j$ and we use $Sp(V_j)$ to denote the set of all spouses of $V_j$. A path from $V_s$ to $V_e$ is a directed path if for any pair of adjacent nodes $V_i$ and $V_j$ in the path where $i\le j$, $V_i$ is a parent of $V_j$. $V_s$ is an ancestor of $V_e$ if there is a directed path from $V_{s}$ to $V_{e}$. The ancestor set of $V_e$ is denoted as $An(V_e)$.

A DAG is a directed graph without directed cycles (See an example of directed cycle in Figure~\ref{Fig:cyce_dag_mag} (a)). When a DAG satisfies the following Markov condition and faithfulness assumption, we can read dependency/independency in a distribution from the DAG.

\begin{figure}
\centering
  \includegraphics[width=2.5in]{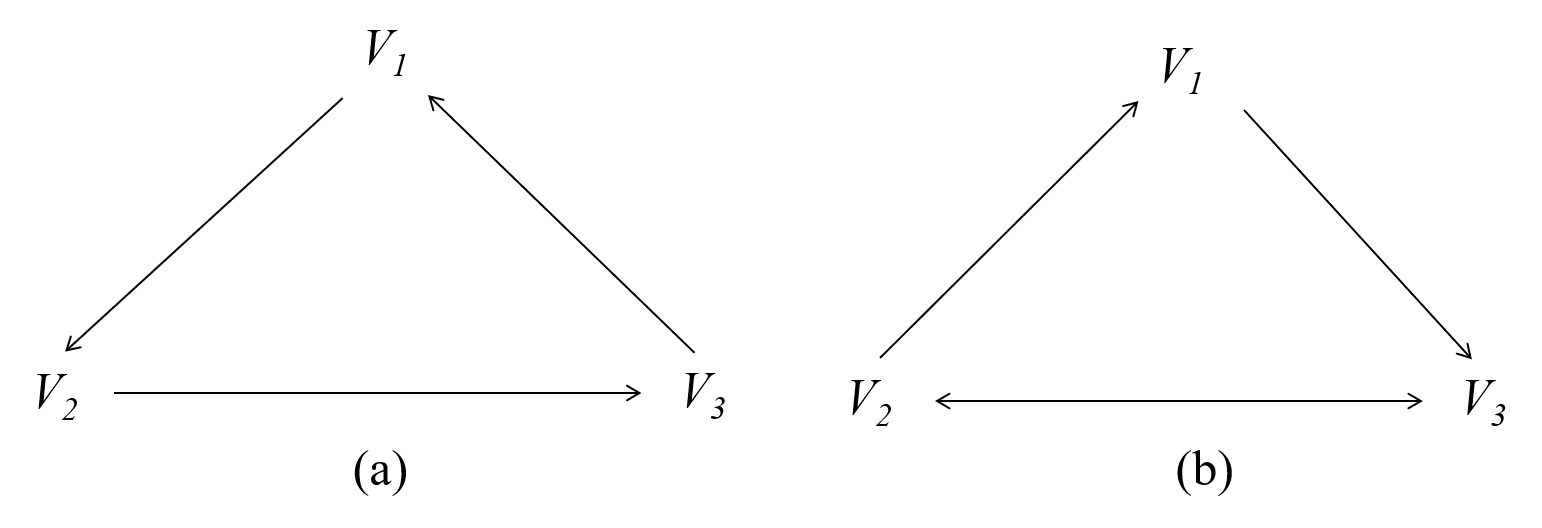}
\caption{(a) An example of a directed cycle, (b) An example of an almost directed cycle.}
\label{Fig:cyce_dag_mag}
\end{figure}

\begin{definition} [Markov condition~\cite{pearl2009causality}]
\label{Markov condition}
Given a DAG $\mathcal{G}=(\mathbf{V}, \mathbf{E})$ and $\textsl{P}(\mathbf{V})$, the joint probability distribution of $\mathbf{V}$, $\mathcal{G}$ satisfies the Markov condition if for $\forall V_i \in \mathbf{V}$, $V_i$ probabilistically independent of all non-descendants of $V_i$, given the parents of $V_i$.
\end{definition}

Based on the Markov condition, according to $\mathcal{G}$  $\textsl{P}(\mathbf{V})$ can be factorized into: $\textsl{P}(\mathbf{V}) = \prod_i p(V_i |Pa(V_i))$.

\begin{definition}[Faithfulness~\cite{spirtes2000causation}]
\label{Faithfulness}
A DAG $\mathcal{G}=(\mathbf{V}, \mathbf{E})$ is faithful to $\textsl{P}(\mathbf{V})$ iff every independence presenting in $\textsl{P}(\mathbf{V})$  is entailed by $\mathcal{G}$ and fulfills the Markov condition. A distribution $\textsl{P}(\mathbf{V})$ is faithful to a DAG $\mathcal{G}$ iff there exists a DAG $\mathcal{G}$ which is faithful to $\textsl{P}(\mathbf{V})$.
\end{definition}

A \emph{causal DAG} is a DAG in which a node's parents are interpreted as its direct causes. To learn a causal DAG from observational data, we also need to assume causal sufficiency.

\begin{definition}[Causal sufficiency~\cite{spirtes2000causation}]
\label{Causalsufficiency}
A dataset satisfies causal sufficiency if for every pair of variables $(V_i, V_j)$ in $\mathbf{V}$, all their common causes are also in $\mathbf{V}$.
\end{definition}

When causal sufficiency is violated, ancestral graphs are used to represent data generating processes. An \emph{ancestral graph} $\mathcal{M}$ is a mixed graph that does not contain directed cycles or almost directed cycles. A graph is called a \emph{mixed graph} if it includes directed and bi-directed edges. When there exists a bi-directed edge $V_{i} \leftrightarrow V_{j}$ in $\mathcal{M}$, there is not directed path $V_{i} \rightarrow V_{j}$ or $V_{i}\leftarrow V_{j}$. An \emph{almost directed cycle} occurs when $V_i\leftrightarrow V_j$ is in $\mathcal{M}$ and $V_j\in An(V_i)$, Figure~\ref{Fig:cyce_dag_mag} (b) is such an example. Let ``$\ast$'' denote any allowed edge marks. For a path $\pi$, a non-ending node $V_{i}$ is a \emph{collider} on $\pi$ if $\pi$ contains $V_{i-1}\astrightarrow V_{i} \astleftarrow V_{i+1}$. A path $\pi$ is a \emph{collider path} if each node excluding the ending nodes on $\pi$ is a collider.
\begin{definition}[m-separation~\cite{richardson2002ancestral}]
 \label{m-separation}
In an ancestral graph $\mathcal{G}$, a path $\pi$ between $V_{i}$ and $V_{j}$ is said to be m-separated by a set of nodes $\mathbf{Z}$ (possibly $\emptyset$) if (1). $\pi$ does not contain any collider which is in $\mathbf{Z}$, or (2). for any collider $V_i$ on the path $\pi$, $V_i \notin \mathbf{Z}$ and no descendant of $V_i$ is in $\mathbf{Z}$.
Two nodes $V_{i}$ and $V_{j}$ are said to be m-connected by $\mathbf{Z}$ in $\mathcal{G}$ if $V_{i}$ and $V_{j}$ are not m-separated by $\mathbf{Z}$.
\end{definition}

If $V_i$ and $V_j$ are m-separated by $\mathbf{Z}$, $V_i \indep V_j |\mathbf{Z}$, the information flow from $V_i$ and $V_j$ is ``blocked'' by $\mathbf{Z}$. In this paper, we call $\mathbf{Z}$ a block set of the ordered pair ($V_i, V_j$).

\begin{definition}[Maximal ancestral graph (MAG)]
\label{MAG}
An ancestral graph $\mathcal{M}=(\mathbf{V}, \mathbf{E})$ is called a maximal ancestral graph if every pair of non-adjacent nodes $V_{i}$, $V_{j}$ in $\mathcal{M}$ can be m-separated by a set $\mathbf{Z}\subseteq \mathbf{V}\backslash \{V_{i}, V_{j}\}$.
\end{definition}

We present an example of causal MAG in Figure~\ref{Fig:causaldiagram} (b). Figure~\ref{Fig:causaldiagram} (a) is its corresponding causal DAG with the hidden variables $U_1$ and $U_2$ shown. In a MAG $\mathcal{M}$, if a path between $V_s$ and $V_e$ is a directed path, it is called a \emph{causal path} from $V_s$ to $V_e$. A non-directed path between $V_s$ and $V_e$ is a \emph{non-causal path}. In Figure~\ref{Fig:causaldiagram}(b), path $W\rightarrow Y$ is a causal path, and other paths between $W$ and $Y$ are non-causal paths.

\begin{figure}
\centering
  \includegraphics[width=3.2in]{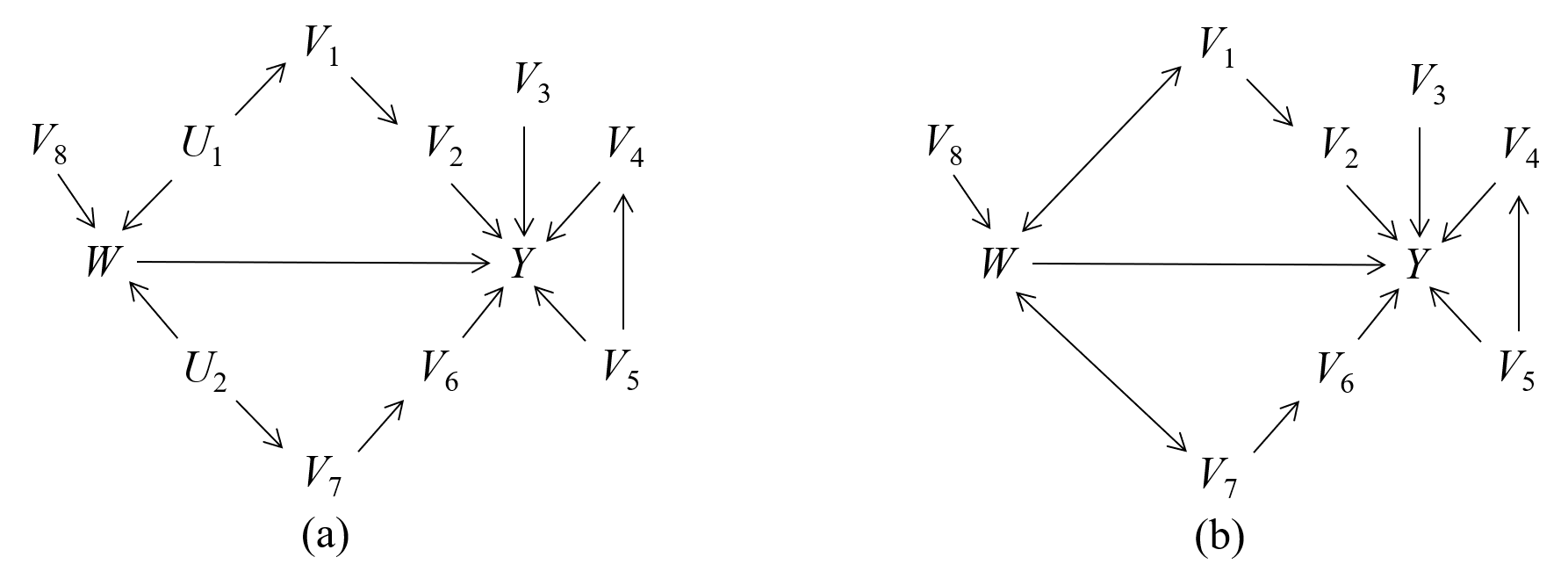}
\caption{(a) A DAG with two hidden variables $U_1$ and $U_2$. (b) The corresponding MAG of the DAG.}
\label{Fig:causaldiagram}
\end{figure}

\begin{definition}[Visibility \cite{zhang2008causal}]
\label{Visibility}
Given a MAG $\mathcal{M}$, a directed edge $V_{i}\rightarrow V_{j}$ is visible if there is a node $V_{k}$ not adjacent to $V_{j}$, such that either there is an edge between $V_{k}$ and $V_{i}$ that is into $V_{i}$, or there is a collider path between $V_{k}$ and $V_{i}$ that is into $V_{i}$ and every node in this path is a parent of $V_{j}$. Otherwise, $V_{i}\rightarrow V_{j}$ is said to be invisible.
\end{definition}

\begin{figure}
\centering
  \includegraphics[width=2.5 in]{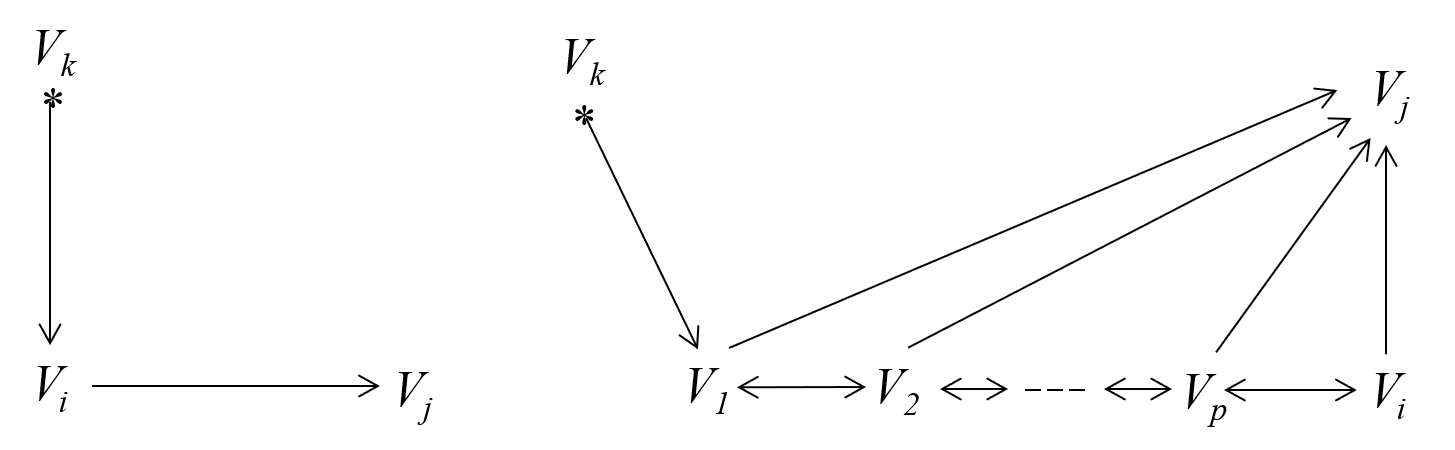}
\caption{Two possible configurations of the visible edge for $V_i \rightarrow V_j$.}
\label{Fig:visible_edge}
\end{figure}

Figure~\ref{Fig:visible_edge} shows two different graphical configurations where the edge $V_i \rightarrow V_j$ is visible. In a given MAG $\mathcal{M}$, a visible edge $V_{i}\rightarrow V_{j}$ implies that there are no hidden variables between $V_{i}$ and $V_{j}$. Otherwise, the edge $V_{i}\rightarrow V_{j}$ may include hidden variables.

\subsection{Adjustment set for causal effect estimation}
Let $W$ be a binary treatment variable, $Y$ the outcome variable. The causal effect of $W$ on $Y$, $CE(W, Y)$ is the change of $Y$ due to a change of $W$. When we use the \emph{do} operator~\cite{pearl2009causality} to represent a manipulation of a variable (e.g. set $W$ to 1) and denote $W=1$ and $W=0$ as $w, w'$ respectively, $CE(W, Y)$ is defined as:

\begin{equation}\label{eq1}
CE(W, Y) = \mathbb{E}(Y\mid do~(w)) -  \mathbb{E}(Y\mid do~(w'))
\end{equation}

\noindent where $\mathbb{E}$ is the Expectation. $CE(W, Y)$ can be found in an experimental setting where $W$ is physically set as 0 and 1, respectively. Let us assume that an underlying mechanism dictates the causal relationships among variables, and $CE(W,Y)$ is determined by the mechanism.

When estimating the causal effect of W on Y, the effect of other variables needs to be eliminated or adjusted to obtain an unbiased estimation of the causal effect. In graphical causal modelling, the back-door criterion is the main principle for identifying a set of adjustment variables, denoted as $\mathbf{Z}$ in this paper.

\begin{definition}[Back-door criterion in DAG]
A set of variables $\mathbf{Z}$ satisfies the back-door criterion relative to $(W, Y)$ in a DAG $\mathcal{G}$ if (1) $\mathbf{Z}$ does not contain descendants of $W$; (2) $\mathbf{Z}$ blocks every back-door path between $W$ and $Y$ (i.e. the paths with an arrow pointing to $W$).
\end{definition}

Given an adjustment set $\mathbf{Z}$, the causal effect $CE(W, Y)$ can be estimated unbiasedly as follows:
\begin{equation}
\label{eq2}
CE(W, Y) =\sum_{\mathbf{z}}[\mathbb{E}(Y\mid w, \mathbf{Z}=\mathbf{z})- \mathbb{E}(Y\mid w', \mathbf{Z}=\mathbf{z})]p(\mathbf{Z}=\mathbf{z})
\end{equation}

When the assumption of causal sufficiency does not hold, the generalised back-door criterion in a MAG can be used to search for an appropriate adjustment set $\mathbf{Z}$ for estimating $CE(W, Y)$~\cite{maathuis2015generalized}.
\begin{definition}[Back-door path in MAG]
\label{back door path MAG}
For the ordered pair $(W, Y)$ in a MAG, a path from $W$ to $Y$ is a back-door path if it does not have a visible edge out of $W$.
\end{definition}

\begin{definition}[Generalised back-door criterion for MAG]
\label{generalized back door criterion for MAG}
Given a MAG $\mathcal{M}=(\mathbf{V}, \mathbf{E})$, a set $\mathbf{Z}\subseteq \mathbf{V}\setminus \{W, Y\}$ satisfies the generalised back-door criterion w.r.t. ($W, Y$) in $\mathcal{M}$ if (1) $\mathbf{Z}$ does not contain descendants of $W$; (2) $\mathbf{Z}$ blocks every back-door path between $W$ and $Y$.
\end{definition}

If a set $\mathbf{Z}$ in $\mathcal{M}$ satisfies the generalised back-door criterion, relative to $(W, Y)$, then the unbiased estimation of $CE(W, Y)$ can be achieved by Eq.(\ref{eq2}).

In some MAGs, the adjustment sets could not be found due to the causal ambiguities, and this can be judged based on the amenability of a MAG~\cite{van2014constructing} as defined below.
\begin{definition}[Amenable MAG w.r.t. ($W, Y$)]
\label{Amenability}
Let $W$ and $Y$ be two nodes in a MAG $\mathcal{M}=(\mathbf{V}, \mathbf{E})$. The MAG is adjustment amenable w.r.t. $(W, Y)$ if $W\rightarrow Y$ is visible.
\end{definition}

If the MAG $\mathcal{M}$ is not adjustment amenable w.r.t. $(W, Y)$, then no adjustment set $\mathbf{Z}$ in $\mathbf{V}\setminus \{W, Y\}$ can be found~\cite{van2014constructing}.

\section{Local search for an adjustment set}
\label{Sec:Methodology}
\subsection{The Theorem}
\label{Subsec:ACTM03}
Let $\mathbf{D}$ be a dataset containing a binary treatment variable $W$, an outcome variable $Y$ which is binary or numerical, and $\mathbf{X}$, a set of all other variables of any type. We assume that the dataset is generated from an underlying causal MAG M which is adjustment amenable to the ordered pair $(W, Y)$. We also assume that all the variables in $\mathbf{X}$ are measured before applying the treatment and observing $Y$, indicating that variables in $\mathbf{X}$ are all non-descendants of $W$ or $Y$ in $\mathcal{M}$. The pretreatment variable assumption is realistic as it reflects what normally happens in practice, that is, all the other variables are often measured before the treatment is applied and the outcome is observed. This assumption is commonly used in~\cite{de2011covariate,hill2011bayesian,imai2014covariate,vanderweele2011new}.

In the proposed Theorem~\ref{Theorem-MAG-PA-W01} below, $\mathcal{M}_{\underline{W}}$ denotes the manipulated $\mathcal{M}$ from which $W\rightarrow Y$ has been removed, and $\mathcal{M}_{\underline{W}}$ is still a MAG by the closure property of a MAG~\cite{richardson2002ancestral}.

\begin{thm} [Adjustment set discovery via local search]
\label{Theorem-MAG-PA-W01}
Given a causal MAG $\mathcal{M}$ which is adjustment amenable to the ordered pair ($W, Y$), there exists at least an adjustment set $\mathbf{Z}\subseteq \mathbf{X}$ for ($W, Y$) such that $\mathbf{Z}$ is a subset of $Adj(W \cup Y)$, where $Adj(W\cup Y)$ is the shorthand of $Adj(W)\cup Adj(Y$) in $\mathcal{M}_{\underline{W}}$.
\end{thm}

\begin{proof}
The requirement that $\mathcal{M}$ is adjustment amenable relative to the ordered pair $(W, Y)$ is to ensure that the edge $W\rightarrow Y$ is visible and the causal effect of $W$ on $Y$ is identifiability. There are two types of paths between $W$ and $Y$, causal paths and back-door paths. In our problem setting, pretreatment variables $\mathbf{X}$ are non-descendants of $W$ and $Y$, and hence $\mathcal{M}_{\underline{W}}$ does not contain any causal path between $W$ and $Y$.

To prove that an adjustment set $\mathbf{Z}$ in the MAG $\mathcal{M}$ is a subset of $Adj(W\cup Y)$, we will show that the set $Pa(Y)\setminus\{W\}$ blocks all back-door paths from $W$ to $Y$. Let $\pi$ be any back-door path between $W$ and $Y$, and let $X$ denote the node adjacent to $Y$ on $\pi$. Note that $X$ is not $W$, for the assumption, $W\rightarrow Y$ is visible and so is not a back-door path. There are two cases as follows.

\textbf{1}. The edge between $X$ and $Y$ is $X\rightarrow Y$. In this case, $X$ is a non-collider on $\pi$ and is in the set $Pa(Y)\setminus\{W\}$, and so $\pi$ is blocked by $Pa(Y)\setminus\{W\}$.

\textbf{2}. The edge between $X$ and $Y$ is not $X\rightarrow Y$. Then it must be bi-directed edge: $X\rightarrow Y$, for otherwise $X$ would be a descendant of $Y$ and hence a descendant of $W$, contradicting the assumption. This means that $X$ is not an ancestor of $Y$ and hence also not an ancestor of $W$ in the MAG $\mathcal{M}$. Then on $\pi$, there must be a collider. Let $C$ be the collider closest to $X$ on $\pi$. Then $C$ must be a descendant of $X$ (possibly $X$ itself), which means that $C$ does not have a descendant in $Pa(Y)\setminus \{W\}$, hence $\pi$ is blocked by $Pa(Y)\setminus \{W\}$.
 
Therefore, given a MAG $\mathcal{M}$, $Pa(Y)\setminus \{W\}$, which is a subset of $Adj(W\cup Y)$, is an adjustment set for the ordered pair $(W, Y)$.
\end{proof}

\subsection{Causal effect estimation}
Based on the above Theorem developed, we can firstly employ a local causal structure discovery algorithm to obtain $Adj(W\cup Y)$ from data, then our answer to a causal query is an Adjustment Set-Causal Effect Table, or ASCET, where each row contains a candidate adjustment set (a subset of $Adj(W\cup Y)$) and the causal effect estimated corresponding to the candidate adjustment set.

\textbf{Example}. Suppose $Adj(W\cup Y)=\{X_1, X_2, X_3\}$, Table~\ref{tab::exampleCE01} shows an example ASCET where 1 denotes a variable is in a candidate adjustment set and 0 otherwise.

\subsection{Removing insignificant adjustment variables}
\label{subsec::insignificantvariables}
We reduce the size of ASCET by removing variables that do not significantly affect causal effect estimation. To test if a variable significantly affects causal effect estimation, we compare the cause effect estimated by including the variable with the cause effect estimated by excluding the variable in all cases. If the average difference is smaller than a threshold, the variable is insignificant. More precisely, we have the following definition.

\begin{minipage}{\textwidth}
\begin{minipage}[t]{0.25\textwidth}
\begin{center}
\small
 \makeatletter\def\@captype{table}\makeatother\caption{An example ASCET.}
\begin{tabular}{|ccc|c|}
\hline
$X_1$ & $X_2$ & $X_3$ &  $CE$ \\ \hline
0 & 0 & 0 & 0.4    \\
1 & 0 & 0 & 0.3     \\
0 & 1 & 0 & 0.4     \\
0 & 0 & 1 & 0.5  \\
1 & 1 & 0 & 0.4   \\
1 & 0 & 1 & 0.1     \\
0 & 1 & 1 & 0.5  \\
1 & 1 & 1 & 0.1  \\
\hline
\end{tabular}
\label{tab::exampleCE01}
\end{center}
\end{minipage}
\begin{minipage}[t]{0.21\textwidth}
\centering
\small
\makeatletter\def\@captype{table}\makeatother\caption{The final ASCET.}
\begin{tabular}{|cc|c|}
\hline
$X_1$ & $X_3$ &  $CE$ \\ \hline
0  & 0 & 0.4    \\
1  & 0 & 0.3     \\
0  & 1 & 0.5  \\
1  & 1 & 0.1     \\
\hline
\end{tabular}
\label{tab::exampleCE02}
\end{minipage}
\end{minipage}

\begin{definition}
\label{Sensitivity}
For each $X\in Adj(W\cup Y)$, let $\mathbf{Z}\subseteq Adj(W\cup Y)$ be a candidate adjustment set containing $X$ and $\mathbf{Z'}= \mathbf{Z}\setminus \{X\}$, the sensitivity of $X$ is defined as:
\begin{equation}
\label{eq3}
Sen(X)=\frac{1}{\rho} \sum_{\mathbf{Z}} |CE_{\mathbf{Z}}-CE_{\mathbf{Z'}}|
\end{equation}
\noindent where $\rho$ is the size of the power set of $Adj(W\cup Y)\setminus\{X\}$, $CE_{\mathbf{Z}}$ and $CE_{\mathbf{Z'}}$ are average causal effects of $W$ on $Y$ with the adjustment sets $\mathbf{Z}$ and $\mathbf{Z'}$ respectively.
\end{definition}

The sensitivity of variable $X$ reflects its impact on the estimated causal effect of $W$ on $Y$. When $X$ has a low sensitivity, no significant error will be introduced if we exclude it from $Adj(W\cup Y)$. So, $X$ is removed from $Adj(W\cup Y)$, and the size of ASCET is reduced.

\textbf{Example}. We use the example in Table~\ref{tab::exampleCE01} to illustrate how we reduce the size of ASCET. We use $X_1$ as an example, $\rho = 4$ as the size of the power set of $\{X_2, X_3\}$ is 4, then $Sen(X_1)= \frac{1}{4} (|CE_{\{X_1\}}-CE_{\phi}|+|CE_{\{X_1, X_2\}} - CE_{\{X_2\}}|+|CE_{\{X_1, X_3\}}-CE_{\{X_3\}}|+|CE_{\{X_1,X_2, X_3\}} - CE_{\{X_2, X_3\}}|) = \frac{1}{4}(0.1+0+0.4+0.4)=0.25$. In the same way, $Sen(X_2)=0.025$ and $Sen(X_3)=0.175$. Hence, $X_2$ has a low sensitivity such that we can exclude it from $Adj(W\cup Y)$ and the corresponding rows in ASCET can be removed. The final result is presented as Table~\ref{tab::exampleCE02}.

\subsection{The proposed DICE algorithm}
\label{Subsec:ACTM04}
Algorithm~\ref{pseudocode01} presents our proposed algorithm for Data-drIven Causal Effect estimation without causal sufficiency (DICE). DICE contains two parts. Part 1 (lines 1 to 13) is for finding $Adj(W\cup Y)$. Firstly, a local causal structure learning algorithm (such as PC-Select~\cite{buhlmann2010variable} or HITON-PC~\cite{aliferis2010local}) is used to find the parents of $W$ and $Y$. In this work, we use PC-Select, given its high accuracy, PC-Select is a local version of PC~\cite{spirtes2000causation}, an algorithm for learning a global causal structure. Then we check whether $Adj(W) = Adj(W)\setminus \{Y\}$ is empty to determine if the causal effect of $W$ on $Y$ is identifiable or not from data. If $Adj(W)$ is empty, no adjustment set can be found, and DICE returns an empty ASCET and terminates.

Part 2 of Algorithm~\ref{pseudocode01} (lines 7 to 16) firstly estimates each possible causal effect when adjusting each subset of $Adj(W \cup Y)$. We use Propensity Score Matching (PSM)~\cite{rosenbaum1983central} to estimate causal effects where propensity score is calculated by \emph{glm} and matching is performed by \emph{Match} in the $\mathbb{R}$ packages \emph{stats} and \emph{Matching}~\cite{sekhon2008multivariate}, respectively. The adjustment sets and the corresponding causal effects estimated are added to the ASCET. Next, DICE calculates the sensitivity of each variable in $Adj(W\cup Y)$, and if a variable's sensitivity is below the given threshold, DICE removes all the adjustment sets containing the variable and the corresponding causal effects from the ASCET.

\textbf{Time-complexity analysis.} PC-Select and PSM contribute the most to the time complexity of DICE. In the worst case, the complexity of PC-Select is $\mathbf{O}(p2^{p-1}n)$~\cite{buhlmann2010variable}, where $p$ is the number of variables and $n$ the number of samples. In our problem, the worst case is when all variables of $\mathbf{X}$ are causes of $W$ and/or $Y$, which, however rarely occurs. In most cases, PC-Select can handle thousands of variables~\cite{buhlmann2010variable}. The time complexity of PSM is $\mathbf{O}(p^{2})$. Our experiment in Section~\ref{sec:experiments} shows that DICE scales with both $p$ and $n$ well.

\begin{algorithm}[tp]
\caption{Data-drIven Causal Effect estimation without causal sufficiency (DICE)}
\label{pseudocode01}
\noindent {\textbf{Input}}: Dataset $\mathbf{D}$ with treatment $W$, pretreatment variables $\mathbf{X}$, outcome $Y$ and a sensitivity threshold $\tau$.\\
\noindent {\textbf{Output}: The ASCET.}
\begin{algorithmic}[1]
\STATE {Call a local causal structure learning algorithm to find $Adj(W)\setminus \{Y\}$ and $Adj(Y)\setminus \{W\}$ from $\mathbf{D}$, respectively.}
\STATE {$ASCET\leftarrow \emptyset$}
\IF {$Adj(W)$ is empty}
\STATE{Adjustment set could not be found.}
\RETURN{$ASCET$.}
\ELSE
\FOR {each subset $\mathbf{Z} \subseteq Adj(W \cup Y)$}
\STATE {calculate CE($W, Y$) by adjusting $\mathbf{Z}$.}
\STATE{add pair $(\mathbf{Z}, CE(W, Y))$ to ASCET.}
\ENDFOR
\STATE{For each $X\in Adj(W \cup Y)$ calculate $Sen(X)$.}
\IF{$Sen(X)<\tau$}
\STATE{remove from ASCET all the rows containing $X$.}
\ENDIF
\RETURN{ASCET.}
\ENDIF
\end{algorithmic}
\end{algorithm}

\section{Experiments}
\label{sec:experiments}
\subsection{The quality of local structure learning}
\label{subsec::Performance}
Before evaluating DICE, we assess the quality of local structure learning algorithm, i.e. PC-Select used in our paper versus global structure learning algorithms. The synthetic datasets are generated according to the underlying causal DAG in Fig~\ref{Fig:causaldiagram} (a) with 8 pretreatment variables, 2 hidden variables, $W$ and $Y$ by following the procedure in~\cite{haggstrom2018data}. The underlying causal MAG is presented in Fig~\ref{Fig:causaldiagram} (b), i.e. $U_1$ and $U_2$ are removed from these generated datasets.

\begin{figure}[t]
\centering
\begin{minipage}[t]{0.5\linewidth}
\includegraphics[width=1.7in, height=1.0in]{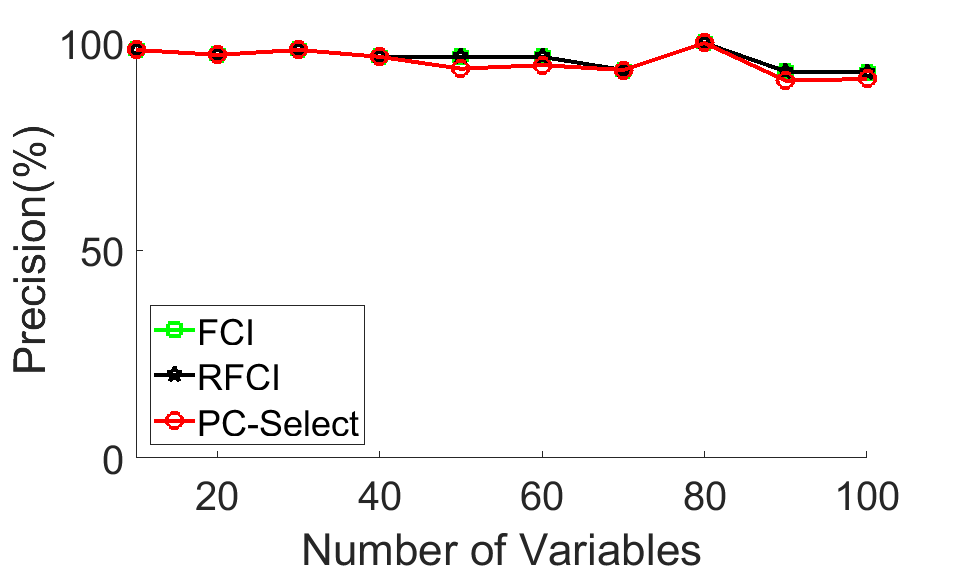}
\end{minipage}%
\hfil
\begin{minipage}[t]{0.5\linewidth}
\includegraphics[width=1.7in, height=1.0in]{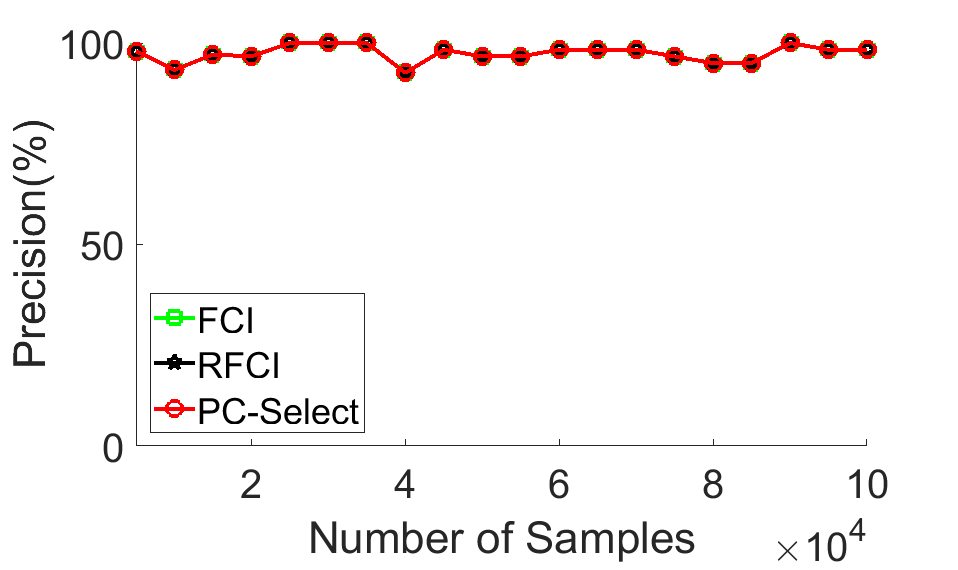}
\end{minipage}%
\vfil
\begin{minipage}[t]{0.5\linewidth}
\includegraphics[width=1.7in, height=1.0in]{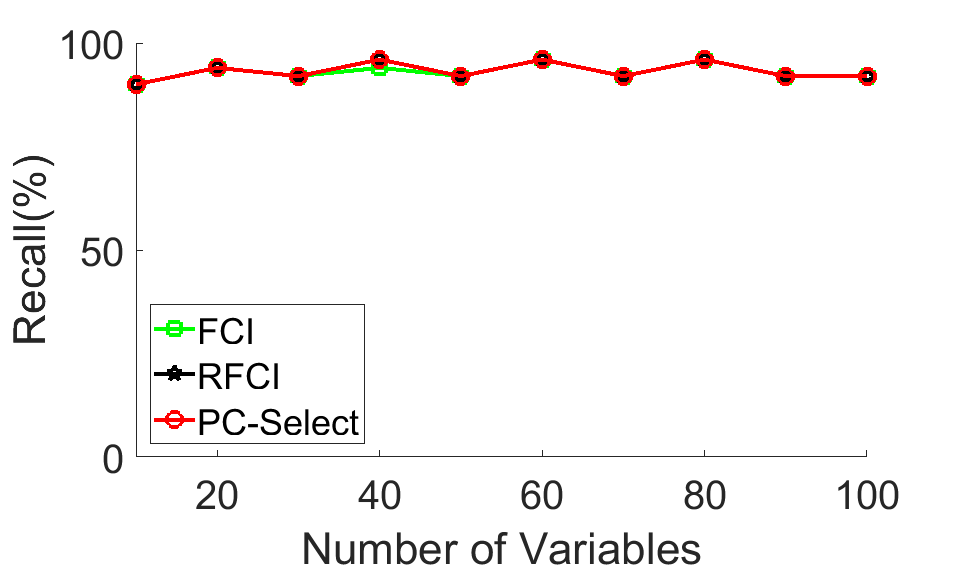}
\end{minipage}%
\hfil
\begin{minipage}[t]{0.5\linewidth}
\includegraphics[width=1.7in, height=1.0in]{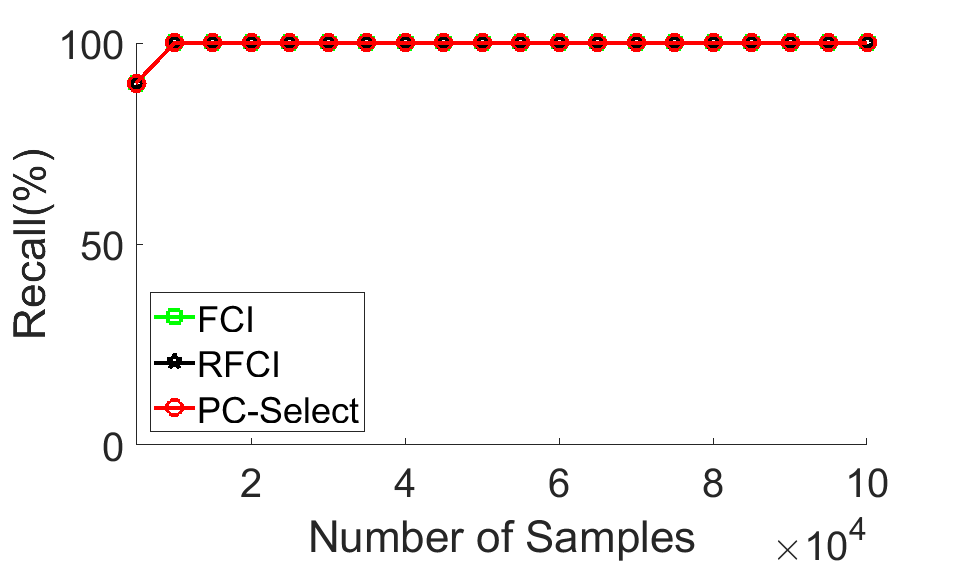}
\end{minipage}%
\vfil
\begin{minipage}[t]{0.5\linewidth}
\includegraphics[width=1.7in, height=1.0in]{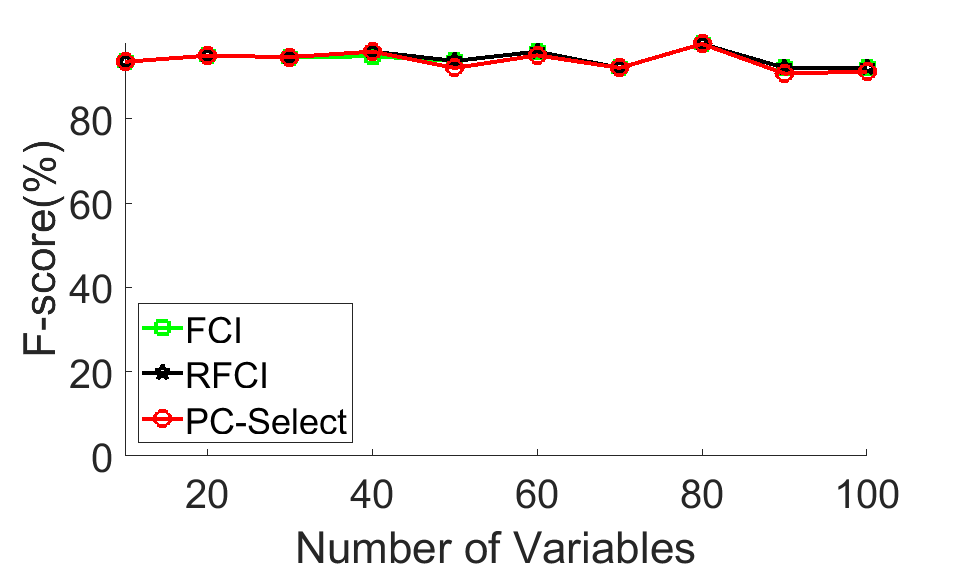}
\end{minipage}%
\hfil
\begin{minipage}[t]{0.5\linewidth}
\includegraphics[width=1.7in, height=1.0in]{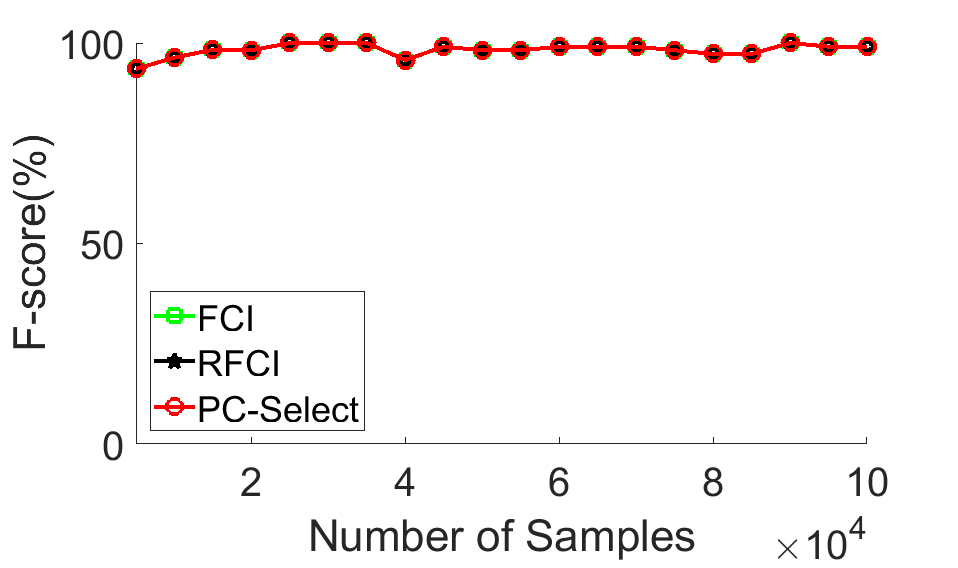}
\end{minipage}
\caption{The Precision, Recall and F-score of FCI \& RFCI \& PC-Select algorithms on two groups of synthetic datasets w.r.t. the number of variables and samples. The qualities of parent discovery by the three algorithms are very consistent. but the time efficiencies of the algorithms are quite different as shown in Figure~\ref{fig:resultsscaleonsyntheticdataset}}
\label{FIGURE_PRECISION}
\end{figure}

We generate two groups of datasets to evaluate the impact of sample size and the number of variables on the quality of the local structure learning algorithm, respectively. The first group of datasets include 10 datasets , each containing 5K samples for 10 variables generated from Fig~\ref{Fig:causaldiagram}(b), plus 0, 10, 20, $\dots$, 90 random variables, respectively. The second group of datasets include 20 datasets with 10 fixed variables of the MAG, but with varying sample sizes of 5K, 10K, $\dots$, 95K and 100K. These datasets with hidden variables are used to evaluate the quality of the local structure learning algorithm.

We use two global structure learning algorithms, Fast causal inference (FCI)~\cite{spirtes2000causation} and Really FCI (RFCI)~\cite{colombo2012learning} to learn from each of the generated datasets a PAG (Partial Ancestral Graph) which represents a Markov equivalence class of MAGs encoding the same dependence/independence relations in the data. Then we extract $Adj(W)$ and $Adj(Y)$ from the learned PAG. For PC-Select, we apply it twice to a dataset, one having $W$ as the target and the other having $Y$ as the target to learn $Adj(W)$ and $Adj(Y)$. The implementations of FCI, RFCI and PC-Select are from the $\mathbb{R}$ package \emph{pcalg}~\cite{kalisch2012causal}, with the default parameter settings and a significance level ($\alpha$) of 0.05.

\textbf{Results.}
We perform the experiments 10 times on each of the 30 datasets. We draw the mean results (precision, recall and F-score) of the three algorithms in Figure \ref{FIGURE_PRECISION}. The local structure learning algorithm PC-Selects achieves similar performance as the global structure learning algorithm in $Adj(W)$ and $Adj(Y)$, indicating that adjustment set discovery in data by local search is reliable.

\subsection{Evaluating DICE with real-world datasets}
\label{subsec::realworld}
We evaluate the effectiveness of DICE on three real-world datasets: Jobs training (Jobs for short)~\cite{lalonde1986evaluating}, IHDP~\cite{hill2011bayesian} and Twins~\cite{almond2005costs}. A brief description of the datasets is shown in Table~\ref{tab:detailsreal}.

\begin{table}
\centering
\caption{Summary of the real-world datasets.}
\label{tab:detailsreal}
\small
\begin{tabular}{cccccc}
\hline
Name & \#treated & \#control & \#samples & \#variables   \\
\hline
Jobs & 297 & 2915 & 3214 & 9 \\
IHDP & 139 & 608 & 747 & 24 \\
Twins & 3275 & 1546 & 4821 & 40 \\
\hline
\end{tabular}
\end{table}

\begin{table}
\caption{Results on Jobs. Bias (\%) is the relative error comparing to the ground-truth causal effect. The estimation with the lowest bias is highlighted in each category.}
\label{tab:Jobs}
\small
\begin{tabular}{llll}
  \hline
 Method & ATT & Bias (\%)  & Remarks \\
 \hline
LV-IDA  & -6645.1 & 827.4\%  &  The best of 4 estimates\\
DICE    & \textbf{834.7}  & \textbf{5.8}\% &  The best of 128 estimates \\
\hline
DICE    & 1745.5   & 97.0\%   &  The most probable estimate \\
PSM    & -947.6   & 201.0\% &   \\
CBPS    & 423.3    & 52.0\%    & \\
CFR     & \textbf{742.0}    & \textbf{16.0\%}      & \\
LASSO   & -475.3   & 153.6\% & \\
BART    & -245.2   & 127.7\% & \\
TMLE    & -1901.4  & 314.6\% & \\
CF      & -4438.4  & 600.9\%   &  The best of 30 forests\\
\hline
\end{tabular}
\end{table}

\begin{table}
\caption{Results on IHDP w.r.t. CE and Bias.}
\label{tab:ihdp}
\small
\begin{tabular}{llll}
  \hline
 Method & CE & Bias(\%) & Remarks \\
 \hline
 LV-IDA  & 3.97 & 9.2\% & The best of 300 estimates\\
 DICE & \textbf{4.41} & \textbf{0.7}\% & The best of 16 estimates \\
\hline
 DICE & \textbf{4.48} & \textbf{2.5}\% & The most probable estimate\\
 PSM  &  4.95 & 13.2\% & \\
 CBPS & 4.56 &  4.3\% & \\
 CFR  & 4.86 &   11.0\%  &\\
 LASSO &  5.06   & 15.7\%  & \\
 BART  & 4.69 & 7.3\%  & \\
 TMLE & 4.86  & 11.2\%  & \\
 CF   &  4.25 & 2.9\%   & The best of 30 forests \\
\hline
\end{tabular}
\end{table}

Jobs consists of the original LaLonde dataset (297 treated samples and 425 control samples)~\cite{lalonde1986evaluating} and the Panel Study of Income Dynamics (PSID) observational group (2490 control samples)~\cite{imai2014covariate}. Each sample has 9 variables, including 7 pretreatment variables (which are age in years, schooling in years, indicators for black and Hispanic, marriage status, school degree, previous earnings in 1974 and 1975, and whether the 1974 earnings variable is missing), employment status with/without job training as treatment variable, and 1978 earning as the outcome variable. Because the dataset contains records of people taking part in the training only, as in~\cite{lalonde1986evaluating}, we estimate Average Treatment Effect on Treated (ATT) for DICE and all comparisons, against the ground truth ATT which is \$886~\cite{imai2014covariate}.

IHDP is related to the Infant Health and Development Program (IHDP) on low-birth-weight premature infants~\cite{hill2011bayesian}. IHDP has 24 pretreatment variables and we follow the method in~\cite{hill2011bayesian} to generate the simulated outcome with the ground-truth of 4.38 for the causal effect (CE) by the $\mathbb{R}$ package \emph{npci}\footnote{https://github.com/vdorie/npci}.

\begin{table}
\caption{Results on Twins w.r.t. CE and Bias..}
\label{tab:twins}
\small
\begin{tabular}{llll}
  \hline
 Method & CE & Bias (\%)  & Remarks \\
\hline
 LV-IDA  & -0.007 & 72\%  & The best of 36 estimates\\
 DICE  & \textbf{-0.027}& \textbf{8.0}\% & The best of 64 estimates \\
 \hline
 DICE  & -0.039 & 59.2\% & The most probable estimate \\
 PSM  & -0.018  &  56.7\% &  \\
 CBPS & -0.010 &  59.4\% &  \\
 CFR  & -0.011 & 56.3\% & \\
 LASSO  & -0.016  & 36.2\% & \\
 BART & -0.010  & 41.8\% & \\
 TMLE & -0.016 & 32.9\% & \\
 CF   & \textbf{-0.017} &  \textbf{32.6\%} & The best of 30 forests \\
\hline
\end{tabular}
\end{table}

The Twins dataset consists of samples from twin births in the USA between 1989 and 1991 with the birth weight less than 2000g~\cite{almond2005costs}. We eliminate records with missing values and have 4821 twin pairs left. The treatment variable is birth weight: W=1 for a baby who is heavier in a twin; W=0 otherwise. The mortality after one year is the true outcome for each twin, and the ground-truth causal effect is -0.025. To simulate an observational study, according to~\cite{louizos2017causal}, we use the following Bernoulli distribution to randomly select one of the two twins as the observation and hide the other: $W_i|x_i\sim Bern(sigmoid(\beta^{T}\mathbf{x}+\varepsilon))$, where $x_i$ denotes the pre-treatment variables of the sample $i$, $\mathbf{x}$ denotes the samples set, and $\beta^{T}\sim\mathcal{U}((-0.1,0.1)^{40\times1})$ and $\varepsilon\sim\mathcal{N}(0,0.1)$.

\begin{figure*}[t]
\centering
\includegraphics[width=5in, height=2.35in]{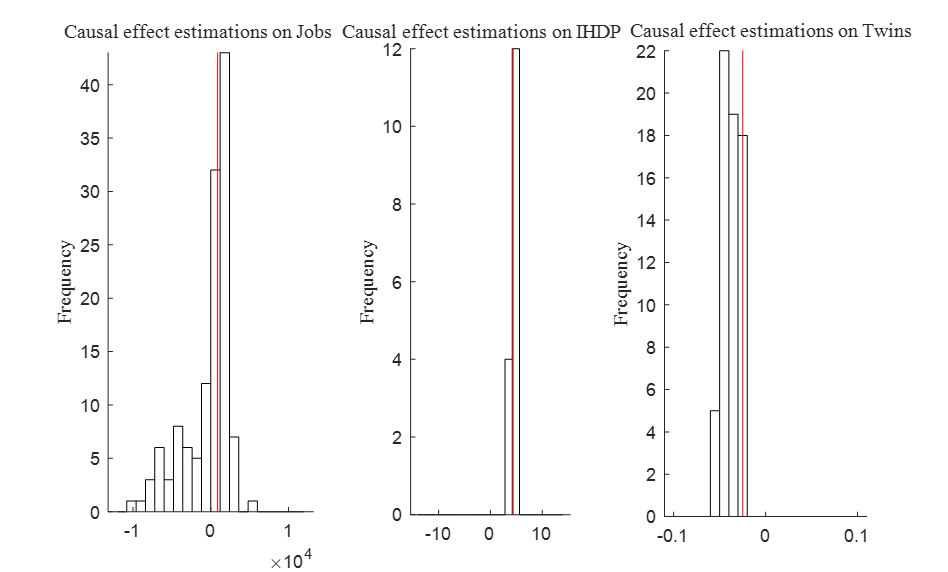}
\caption{Histograms of causal effects in ASCETs for Jobs, IHDP and Twins. The red line denotes the ground truth and the bin size is $(\max(Y)/100)$.}
\label{fig:histplot}
\end{figure*}

There are two major objectives for the experiments.

Firstly, we will show that the true adjustment set can be found as shown in Theorem~\ref{Theorem-MAG-PA-W01}, and therefore the smallest bias of DICE will be very small. We compare DICE with a data-driven causal effect estimation method, LV-IDA~\cite{malinsky2016estimating}\footnote{https://github.com/dmalinsk/lv-ida} which uses FCI to learn a PAG and searches for an adjustment set in each MAG enumerated from the PAG. LV-IDA returns a set of causal effects, one for the adjustment set discovered from each MAG. In the experiment, we report the best estimates for both methods to show the quality of adjustment sets found.

Secondly, we will show that a causal effect estimated by DICE is comparative with the best estimates by methods using domain knowledge. Some statistical and machine learning methods are available for causal effect estimation, with the assumptions of known covariate set and unfoundedness. The typical methods include \textbf{PSM}, propensity score matching with logistic regression~\cite{rosenbaum1983central}; \textbf{CBPS}, covariate balancing propensity score~\cite{imai2014covariate}; \textbf{CFRNET}, a deep learning framework for counterfactual regression with a theoretical error bound (named as CFR)~\cite{shalit2017estimating}; \textbf{LASSO}, Linear regression with the regularization $\ell_{1}$-norm to predict the factual outcome~\cite{tibshirani1996regression}; \textbf{BART}, Bayesian additive regression tree~\cite{chipman2010bart}, a non-linear model which has been applied for counterfactual inference~\cite{hill2011bayesian}; \textbf{TMLE}, targeted maximum likelihood estimation, a doubly robust method~\cite{van2006targeted}; \textbf{CausalForest}, Random forest regression to estimate causal effect~\cite{wager2018estimation} (\textbf{CF} for short). In the experiments, we compare DICE with all the above mentioned methods.

The three datasets will favour the above methods since the unconfoundedness assumption is satisfied. Since the datasets were obtained from well designed observational studies and the covariates were chosen by domain experts.

DICE returns ASCET, the set of adjustment set and causal effect pairs. When reporting causal effects estimated by DICE, we use the most probable value in the ASCET after removing insignificant variables with the sensitivity threshold of $\tau = 0.1$. We group the causal effect in the ASCET by using the bin size of $(\max(Y)/100)$. The average causal effect in the most frequent bin is used as an estimated causal effect.

The parameter setting of DICE is that the \emph{Match} function of estimate is set to ``ATT'' for Jobs and ``ATE'' for IHDP and Twins as in the prior work in~\cite{hill2011bayesian,shalit2017estimating}. Other parameters are set as the default. For CF, the parameter \emph{num.trees} is set from 10 to 300 with the increment of 10, and the best result is reported. For PSM and CBPS, all variables are included in the adjustment variable set.

The experimental results are shown in Tables~\ref{tab:Jobs}, \ref{tab:ihdp}, \ref{tab:twins}, and Figure~\ref{fig:histplot}. From the results, we have the following conclusions.

Firstly, DICE finds the correct adjustment set through local search. Comparing with the other methods, the lowest biases of DICE are consistently the smallest in the three datasets, and are significantly smaller than the lowest biases of LV-IDA, which searchers adjustment sets by global search.

Secondly, causal effects estimated by DICE are comparative with those by other methods which make use of domain knowledge. For the IHDP dataset, DICE achieves the smallest bias. For the Jobs dataset, DICE is ranked the 3rd based on smallest biases. For the Twin dataset, the largest absolute bias is only 0.015, and hence we consider that the estimates by all methods are similar and good.

To show that those comparison methods do not work in a dataset when the unconfoundedness assumption is not satisfied, we have generated datasets with $M$-bias~\cite{pearl2009myth}. The performance of some methods deteriorates greatly. Implementations of other methods do not work with datasets with more than 200K records. We do not show the results in this paper due to the space limitation.

\subsection{Efficiency evaluation of DICE}
\label{subsec::scale}
As PC-Select is a major contributor to the complexity of DICE, in this section, we firstly evaluate the efficiency of PC-Select, and then the overall efficiency of DICE. The computations were performed on a PC with 2.6 GHz Intel Core i7 and 16GB of memory.\\
\textbf{Efficiency evaluation of the structure learning algorithms.}. We record the running time of the structure learning algorithms on the synthetic datasets introduced in Section 4.1 and draw the mean running time in Figure~\ref{fig:resultsscaleonsyntheticdataset}. PC-Select is faster than FCI and RFCI, and comparing to FCI and RFCI, PC-Select scales with the number of variables and samples very well.

\begin{figure}[ht]
\centering
\begin{minipage}[t]{0.5\linewidth}
\includegraphics[width=1.7in,height=1.0in]{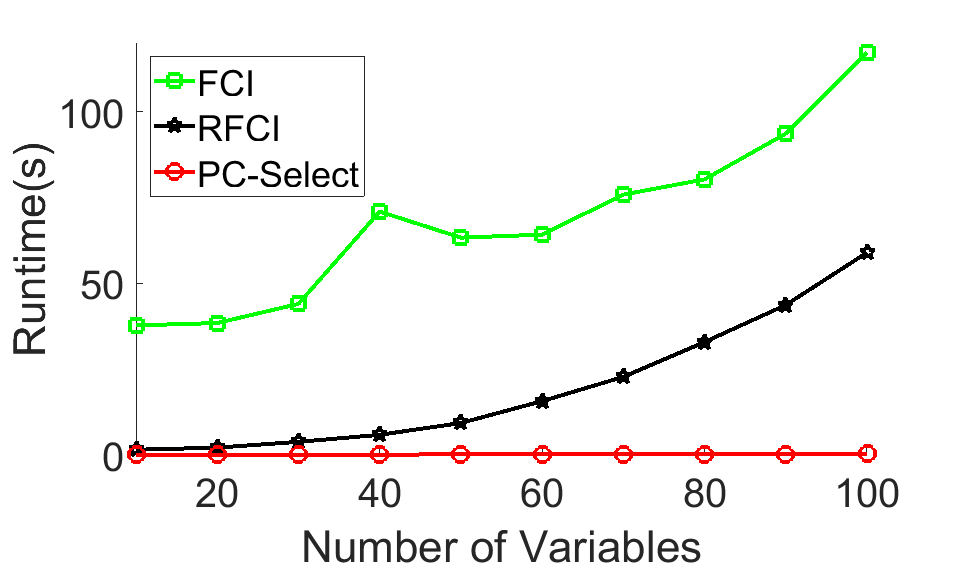}
\end{minipage}%
\hfil
\begin{minipage}[t]{0.5\linewidth}
\includegraphics[width=1.7in, height=1.0in]{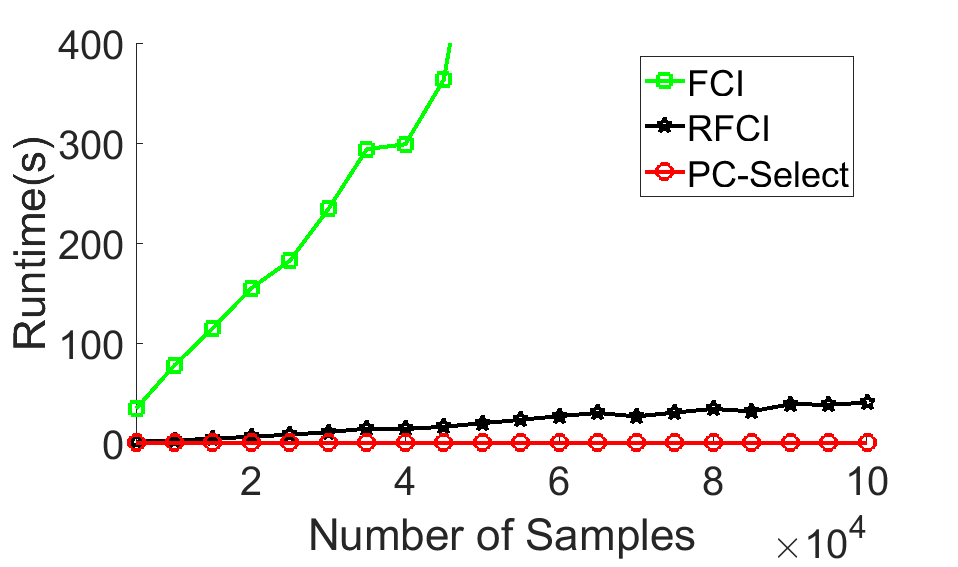}
\end{minipage}
\caption{The runtime on two groups of synthetic datasets w.r.t. the number of variables and samples.}
\label{fig:resultsscaleonsyntheticdataset}
\end{figure}

\begin{figure}[ht]
\scriptsize
\centering
\includegraphics[width=3.2in, height=1.45in]{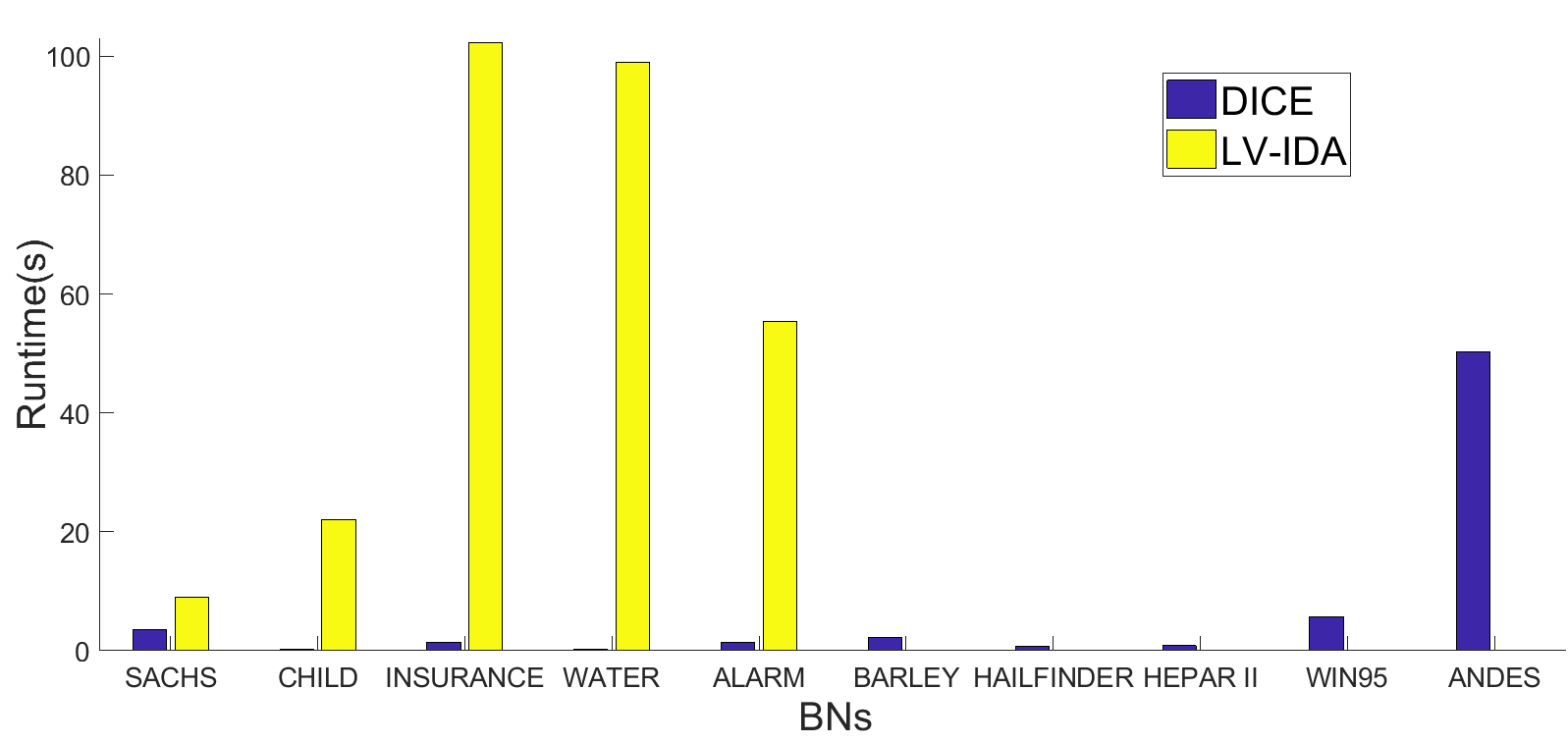}
\caption{The runtime on ten BNs. Note: LV-IDA did not return results in two hours on BARLEY, HAILFINDER, HEPAR II, WIN95 and ANDES.} 
\label{fig:resultsscale}
\end{figure}
\textbf{Experiments on ten standard benchmark BNs}. To evaluate the efficiency of DICE, we use the datasets generated from the ten benchmark BNs from the BN repository\footnote{http://www.bnlearn.com/bnrepository/}: SACHS, CHILD, INSURANCE, WATER, ALARM, BARLEY, HAILFINDER, HEPAR II, WIN95 and ANDES. With each BN, we choose a variable without child nodes as the outcome variable and then select one of its parents without child nodes (except the outcome) as the treatment variable. We generate ten synthetic datasets from the ten BNs with 5000 samples each using the $\mathbb{R}$ package \emph{bnlearn}~\cite{scutari2009learning}. Then we hide 5\% variables (only one variable is hidden for the small BNs SACHS and CHILD) which lie on a non-causal path between the treatment and outcome variables. The same parameter settings are used for DICE and LV-IDA as in Subsection~\ref{subsec::realworld}.

\textbf{Results.} As shown in Figure~\ref{fig:resultsscale}, DICE is faster than LV-IDA across all the 10 datasets. Especially, LV-IDA did not return results for the 5 datasets generated using the larger BNs within two hours, while DICE completed in seconds (at most 50+ seconds) for all datasets.

\section{Conclusion}
Causal query in data is an important means for evidence-based decision making, without relying on or being restricted by domain knowledge. However, its widespread applications are hindered by the low efficiency and unsatisfactory accuracy of existing methods. In this paper, we have developed a theorem which supports the utilisation of efficient local causal structure discovery for causal query and assures the correctness of the result of causal query obtained with local search. Based on the proposed theorem, we have developed DICE to query data for causal effects and the confounders potentially impacting the causal effects. The results returned by DICE provide decision makers not only valuable information on the effect of changing one variable in their systems, but also the awareness of the other factors which could influence the effect. Experiments with synthetic and real-world data have demonstrated the effectiveness and efficiency of DICE. It has been shown that DICE produces better causal effect estimation than LV-IDA, and works on datasets where LV-IDA fails. The causal effects estimated by DICE are as good as the state-of-the-art methods which use domain knowledge. In future, we will evaluate DICE and other methods on purely observational datasets which may include $M$-bias, and apply it to real-world biological applications.

\section{Acknowledgement}
We would like to thank the anonymous reviewers for their
constructive suggestions. We would also like to thank
Assoc. Prof. Jiji Zhang from Lingnan University for many
helpful comments. This research is partially funded by ARC DP170101306 and the National Science Foundation of China (under grant 61876206). The first author is supported by China Scholarship Council.

\bibliography{ecai2020}
\bibliographystyle{ecai}
\end{document}